\documentclass[submission,copyright,creativecommons]{eptcs}
\providecommand{\event}{TARK2021}
\AtBeginDocument{%
  \providecommand\BibTeX{{%
    \normalfont B\kern-0.5em{\scshape i\kern-0.25em b}\kern-0.8em\TeX}}}

\def\titlerunning{\textit{De Re} Updates}
\def\authorrunning{Cohen, Tang, and Wang}

\usepackage{amssymb}
\usepackage{amsmath}
\usepackage{stmaryrd}
\usepackage{comment}
\usepackage[all, knot]{xy}
\usepackage{charter}
\usepackage{color}
\usepackage[english]{babel}
\usepackage{comment}
\usepackage[ruled,vlined]{algorithm2e}
\usepackage{caption}
\newenvironment{proof} {\textsc{Proof}\quad} {\hfill $\blacksquare$\\}
\definecolor{citecolor}{rgb}{0.5,0.5,0.5}

\newcommand{\Ag}{\textbf{I}}
\newcommand{\Var}{\ensuremath{\mathbf{V}}}
\newcommand{\M}{\mathcal{M}}
\newcommand{\N}{\mathcal{N}}

\newcommand{\lr}[1]{\langle #1 \rangle}

\newcommand{\AxTr}{\ensuremath{\mathtt{T}}}

\newcommand{\PR}{\ensuremath{\mathtt{PR}}}
\newcommand{\NM}{\ensuremath{\mathtt{NM}}}

\newcommand{\AxTrans}{\ensuremath{\mathtt{4}}}
\newcommand{\AxEuc}{\ensuremath{\mathtt{5}}}

\newcommand{\Kv}{\ensuremath{\mathsf{Kv}}}

\newcommand{\PAL}{\textsf{PAL}}

\newcommand{\E}{{\mathcal E}}
\newcommand{\F}{{\mathcal F}}

\newcommand{\PALAS}{\textbf{PALAS}}
\newcommand{\DELAS}{\textbf{DELAS}}
\newcommand{\SPALAS}{\textsf{SPALAS}}
\newcommand{\SPALASf}{\textsf{SPALAS5}}
\newcommand{\SDELAS}{\textsf{SDELAS}}
\newcommand{\SDELASf}{\textsf{SDELAS5}}

\newcommand{\DISTK}{\ensuremath{\mathtt{DISTK}}}

\newcommand{\TAUT}{\ensuremath{\mathtt{TAUT}}}

\newcommand{\GENK}{\ensuremath{\mathtt{NECK}}}

\newcommand{\CCOM}{{\texttt{ACOM}}}
\newcommand{\ATOM}{{\texttt{AATOM}}}
\newcommand{\NEG}{{\texttt{ANEG}}}
\newcommand{\CON}{{\texttt{ACON}}}
\newcommand{\ASS}{{\texttt{AASSI}}}
\newcommand{\AK}{{\texttt{AK}}}

\newcommand{\UCCOM}{{\texttt{UCOM}}}
\newcommand{\UATOM}{{\texttt{UATOM}}}
\newcommand{\UNEG}{{\texttt{UNEG}}}
\newcommand{\UCON}{{\texttt{UCON}}}
\newcommand{\UASS}{{\texttt{UASSI}}}
\newcommand{\UAK}{{\texttt{UK}}}

\newcommand{\hK}{\ensuremath{\widehat{\K}}}

\newcommand{\AxEXEAS}{\texttt{DAS}}
\newcommand{\AxDETAS}{\texttt{DETAS}}
\newcommand{\AxKAS}{\texttt{KAS}}
\newcommand{\AxSUBAS}{\texttt{SUBAS}}
\newcommand{\AxSUBP}{\texttt{SUBP}}

\newcommand{\AxRGDP}{\texttt{RIGIDP}}
\newcommand{\AxRGDN}{\texttt{RIGIDN}}

\newcommand{\NECAS}{\ensuremath{\mathtt{NECAS}}}

\newcommand{\AxEFAS}{\ensuremath{\mathtt{EFAS}}}
\newcommand{\AxSUBtoAS}{\texttt{SUB2AS}}
\newcommand{\AxId}{\ensuremath{\mathtt{ID}}}
\newcommand{\AxSym}{\ensuremath{\mathtt{SYM}}}
\newcommand{\AxTranseq}{\ensuremath{\mathtt{TRANS}}}
\renewcommand{\PAL}{\textbf{{PAL}}}
\newcommand{\ELAS}{\textbf{BELAS}}

\newcommand{\SELAS}{\textsf{SBELAS}}
\newcommand{\SELASf}{\textsf{SBELAS5}}

\newcommand{\Nm}{\textbf{N}}

\newcommand{\DEL}{\textbf{DEL}}

\renewcommand{\phi}{\varphi}
\newcommand{\K}{\mathsf{K}}

\newcommand{\lra}{\ensuremath{\leftrightarrow}}

\newcommand{\tot}{\ensuremath{\rightarrowtail}}

\newcommand{\X}{\boldsymbol{X}}
\newcommand{\n}{\boldsymbol{N}}
\newcommand{\Fv}{\mathsf{Fv}}
\newcommand{\var}{\mathsf{Var}}

\newtheorem{theorem}{Theorem}
\newtheorem{definition}[theorem]{Definition}
\newtheorem{lemma}[theorem]{Lemma}
\newtheorem{proposition}[theorem]{Proposition}
\newtheorem{example}[theorem]{Example}

\newtheorem{claim}{Claim}[theorem]
\newtheorem{remark}[theorem]{Remark}

\newenvironment{claimproof}{\par\noindent\textit{Proof of Claim \theclaim:}\space}{\hfill $\blacksquare$}

\renewcommand{\E}{\mathcal{E}}

\renewcommand{\phi}{\varphi}

\DeclareSymbolFont{symbolsC}{U}{txsyc}{m}{n}
\DeclareMathSymbol{\strictif}{\mathrel}{symbolsC}{74}

\newcommand{\Ps}{\textbf{P}}

\renewcommand{\Var}{\textbf{X}}

\newcommand{\AxTrK}{\ensuremath{\mathtt{T}}}

\newcommand{\Va}{\textsf{Var}}
\newcommand{\trs}{\textsf{tr}}

\title{\textit{De Re} Updates}

\author{
Michael Cohen
   \institute{Department of Philosophy\\ Stanford University, USA}
  \email{micohen@stanford.edu}  
\and 
Wen Tang 
  \institute{Department of Philosophy\\ Peking University, China}
\email{1800015421@pku.edu.cn}
\and 
Yanjing Wang\thanks{Corresponding author}
   \institute{Department of Philosophy\\ Peking University, China}
  \email{y.wang@pku.edu.cn}
}

\begin{document}
\maketitle

\begin{abstract}
In this paper, we propose a lightweight yet powerful dynamic epistemic logic that captures not only the distinction between \textit{de dicto} and \textit{de re} knowledge but also the distinction between \textit{de dicto} and \textit{de re} updates. The logic is based on the dynamified version of an epistemic language extended with the assignment operator borrowed from dynamic logic, following the work of Wang and Seligman \cite{Wangnames}. We obtain complete axiomatizations for the counterparts of public announcement logic and event-model-based DEL based on new reduction axioms taking care of the interactions between dynamics and assignments.
\end{abstract}


\section{Introduction}

Epistemic logic is very successful in capturing reasoning patterns of propositional knowledge expressed in terms of \textit{knowing that}. It has been widely applied to formal epistemology, game theory, theoretical computer science, and AI (cf. \cite{ELbook}). 

In particular, the development of dynamic epistemic logic ($\DEL$) provides a flexible framework to formally model how propositional knowledge is communicated and updated by concrete actions and events (cf. e.g., \cite{DELbook}). For example, in public announcement logic (PAL), the \textit{knowing that} modality is equipped by its dynamic counterpart of \textit{announcing that} modality, and the implicit assumptions about agents' ability of obtaining new knowledge are reflected by the interaction of these two modalities in terms of the axioms such as perfect recall and no miracles \cite{MergingJournal09,WC12}. These axioms together with other axioms about the features of updates also give rise to the so-called \textit{reduction axioms}, which can often be used to eliminate the dynamic modalities within the static epistemic logic. 

\subsection{\textit{De re} knowledge and updates}
Despite the great success of the standard epistemic logic of \textit{knowing that}, there are also other commonly used knowledge expressions such as \textit{knowing what/who/how/why} and so on, which were not well-studied in the standard framework. As already observed by Hintikka in the early days of epistemic logic, such expressions are about \textit{knowledge of objects}, or say \textit{de re} knowledge, compared to the \textit{de dicto} knowledge expressed by \textit{knowing that $\phi$} (cf. e.g., \cite{Hintikka95:KAK}). Hintikka pioneered the approach of using first-order (or higher-order) modal logic to capture such expressions \cite{Hintikka2003}, e.g., knowing who murdered Bob can be formalized as $\exists x \K \textit{Kill}(x, Bob)$, in contrast with the \textit{de dicto} knowledge that someone murdered Bob $\K\exists x \textit{Kill}(x, Bob)$.   

Inspired by Hinttika's early idea and discussions in philosophy and linguistics about embedded wh-questions \cite{SG82,stanley2001knowing}, Wang proposed to introduce the \textit{bundle modalities} that pack a quantifier and an epistemic  modality together to capture each \textit{know-wh} as a whole, instead of breaking it down into smaller components \cite{WangBKT}. This leads to a new family of (non-normal) epistemic logics of \textit{know-wh} and new decidable fragments of first-order modal logics \cite{Wang17d,Padmanabha2018}. 

Now a very natural question arises, \textbf{how do we capture the dynamics of such \textit{de re} knowledge?} More specifically, \textbf{can we repeat the success of DEL with a genuine \textit{de re} counterpart?} We hope the present paper can provide positive answers to such questions by presenting a framework which can handle both \textit{de re} and \textit{de dicto} knowledge and updates.

Let us first understand the technical difficulties in handling the \textit{de re} dynamics in the existing framework. To be more specific, consider a logic of \textit{knowing what} featuring the $\Kv$ modality introduced in the very same paper where Plaza invented  the public announcement logic ($\PAL$) \cite{Plaza89:lopc}. Given a non-rigid name $a$, $\Kv_ia$ says that agent $i$ \textit{knows (what) the value/reference of $a$ (is)}. It has a very intuitive semantics induced by its hidden first-order modal form of  $\exists x \K_i (x\approx a)$. After failing to apply the reduction method of $\PAL$, Plaza proposed the question of axiomatizing such a logic with the presence of public announcements. Wang and Fan showed that there is simply no reduction possible in Plaza's language and use a strengthened conditional $\Kv$ modality to axiomatize the logic \cite{WF13,WF14}. Note that although \textit{de dicto} announcements can possibly involve or change \textit{de re} knowledge as nicely shown in \cite{Rendsvig2019termLogicSocialNetwork,OcchipintiLiberman2020},\footnote{For example, announcing \textit{that} Bob knows \textit{what} is the value of the password.} it is not the most natural dynamic counterpart of the $\Kv_i$ operator, as the table below shows: 
\begin{center}
\begin{tabular}{c|c|c}
&knowledge & dynamics\\ 
\hline 
\textit{de dicto} & knowing that  & announcing that \\
\textit{de re} & knowing what  & announcing what 
\end{tabular}
\end{center}
Just like one can know a proposition after an announcement, one should be able to know the value of $a$ after it is announced. However, announcing the value of $a$ cannot be easily handled by \textit{announcing that}, e.g., suppose the domain of $a$ is the set of natural numbers, then you need to use infinitely many non-deterministic announcements in the form of $a\approx k$ for each $k\in\mathbb{N}$. Does it mean we need to introduce constants for all the numbers in the language? What if the value domain is uncountable? 

Given such concerns, a new \textit{announcing value} modality $[a]$ was introduced in \cite{GvEW16} with its dynamic semantics of eliminating the worlds which do not share the same value of $a$ as the designated actual world.\footnote{The modality is called the \textit{public inspection} in \cite{GvEW16}.} However, despite some success of axiomatizing $[a]$ and $\Kv$ in very restricted cases, it remains hard to capture the full logic with know-that, know-what, announce-that,  and announce-what. Unlike the announcement operator, $[a]$ does not obey the \textit{no miracles} axiom,\footnote{A typical \textit{no miracles} axiom is in the shape of $\lr{e}\K\phi\to \K[e]\phi$. It is not valid if $[e]$ is the announcing-what operator.} which is one of the pillars behind dynamic epistemic logic (cf. \cite{WC12}). One reason for this failure is that the \textit{de re} update is \textit{not global}, i.e., the updated effect depends on the value of $a$ on the world where it is executed. It also leads to the problem of not being able to reduce such a dynamic operator. One way to go around is to introduce some rigid constants and all kinds of conditional knowledge operators as in \cite{Baltag16} to eliminate the dynamic operators in a much stronger background logic. However, this seems to be a little bit \textit{ad hoc}, especially when we consider more general \textit{de re} updates which are not entirely public, such as telling $i$ the passwords of $c$ and $d$ but letting the observer $j$ be uncertain about which is which. \textbf{Is there a simpler/natural yet more powerful framework to handle all these dynamics in a uniform way?} Our answer is again affirmative, as to be explained below. 

\subsection{Bridging \textit{de re} and \textit{de dicto} by the assignment operator}
Our main inspiration comes from the treatment of \textit{de re} knowledge using the \textit{assignment operator} $[x:=t]$ from first-order \textit{dynamic logic} \cite{HKT200mit}. The intuitive semantics of $[x:=t]$ is simply an imperative one: assigning variable $x$ the current value of the term $t$. As remarked in \cite{kooi2007dynamic},  Pratt in \cite{Pratt76:PDL} already noticed the connection between the assignment operator and the \textit{$\lambda$-abstraction} that is often used to distinguish the \textit{de re} and \textit{de dicto} readings in first-order modal logic  \cite{STALNAKER1968,Fitting98}.\footnote{There is a large body of research in modal logic trying to  distinguish \textit{de re} and \textit{de dicto} readings, cf. e.g., also \cite{grove1993naming,corsi2013free,HollidayP2014,Rendsvig2010}.} Our technical framework follows the quantifier-free epistemic logic with assignments studied in \cite{Wangnames}, but without considering the termed modalities there. The core idea is that we can use the assignment operator to ``store'' the actual reference of a certain term, and use it under the right scope to capture various forms of \textit{de re} knowledge. For example, $[x:=a]\K_i Px \land \neg \K_iPa$ says that agent $i$ knows of $a$ that $P$, but does not know that $Pa$. As another simple example, note that $[x:=a]\K_i(x\approx a)$ actually expresses that $i$ knows what $a$ is, exactly as  $\Kv_ia$ in the knowing value logic that we mentioned \cite{Plaza89:lopc}. Essentially, the assignment operator can be used as the \textit{bridge} between the \textit{de dicto} and the \textit{de re} knowledge. Now comes the natural question: \textbf{can it also bridge the \textit{de dicto} and \textit{de re} updates? } 

The answer is positive and the solution is surprisingly simple: we just need to use the usual $\DEL$ dynamic operators such as public announcement or event updates together with the assignment operator. For example, the \textit{de dicto} update of announcing that $x\approx a$ can be turned into the \textit{de re} update of announcing the value of $a$ by adding the assignment operator $[x:=a]$ in front of the announcement operator $[!x\approx a]$. As we will show later, the combination of the announcement and the assignment is very powerful and can capture various notions of dependency and conditionals mixing \textit{de re} and \textit{de dicto} updates. The example of telling $i$ the passwords of $c$ and $d$ without letting $j$ know which is which can also be easily handled by using a two-world event model with $x\approx c\land y\approx d$ and $x\approx d\land y\approx c$ as the preconditions respectively, in the scope of two assignment operators $[x:=c][y:=d]$. This will become clear when we introduce the event update formally later on. 

Our treatment of the public announcements and event updates is basically the same as in the standard $\DEL$. Thus the basic properties such as \textit{perfect recall} and \textit{no miracles} between the knowledge operator and dynamic operators stay the same. The combination of the assignment and the dynamics together is responsible for the apparent failure of \textit{no miracles}.\footnote{The non-global nature of \textit{de re} updates comes from the assignments which only record the local value. This is also related to logics of local announcements \cite{Belardinelli_2017} and to the study of opaque updates whose result is not always antecedently known to the agent \cite{Cohen2020-COHOU-2}.} As in the standard \DEL, we will show that the dynamic operators can be eliminated.  

Our contributions in this paper are summarized below: 
\begin{itemize}
    \item We propose a lightweight dynamic epistemic framework with assignment operators, which can handle both \textit{de re} and \textit{de dicto} knowledge and updates in a uniform way.
    \item The public announcement operator and event model update operators can be eliminated qua expressivity as in the standard dynamic epistemic logic.
    \item We obtain complete axiomatizations of all the logics introduced in the paper. 
\end{itemize}
The technical results are relatively straightforward. The main point of the paper is to highlight the use of the assignment operators in capturing the \textit{de re} updates, and propose the alternative static  epistemic logic which can pre-encode \textit{de re} dynamics. The \textbf{magic of the assignment operator} is that it can automatically turn \textit{de dicto} notions into the corresponding \textit{de re} notions almost for free. Therefore we just need to add the assignment operator to a relatively standard \textit{de dicto} epistemic framework to capture all those \textit{de re} knowledge and updates, without introducing various new \textit{ad hoc} modalities.

We hope our framework can also bring new tools for philosophical analysis related to \textit{de re} updates. For example, \textit{de re} updates can be used to analyze scenarios in which an agent has \textit{de dicto} knowledge of every proposition but is still able to learn new \textit{de re} knowledge. Such learning events require \textit{de re} updates. Scenarios in which a propositionally omniscient agent is able to learn something new about their environment  play a central role in philosophy of mind (e.g., Frank Jackson's \textit{Mary's room} thought experiment  \cite{Lewis1979-LEWADD, Jackson1986-JACWMD}). We leave the philosophical discussion to a future occasion.

\paragraph{Structure of the paper} In Section \ref{sec.basics}, we introduce the basic epistemic logic with assignments and its axiomatization. Section \ref{sec.pa} adds the public announcement operator and Section \ref{sec.em} discusses the event model updates with and without factual changes. We conclude with future directions in Section \ref{sec.con}.

\section{Epistemic logic with assignments}\label{sec.basics}
In this section, we present a language of  epistemic logic with assignments. It can be viewed as a simplified version of the language studied in \cite{Wangnames} without the term-modalities. 
\subsection{Language and Semantics}
\begin{definition}[Language of $\ELAS$] Given a set of variables $\Var$, set of names $\Nm$, set of agents $\Ag$ and a set of predicate symbols $\Ps$, the language of Basic Epistemic Logic with Assignments ($\ELAS$) is defined as:
\begin{align*}
t&::= x\mid a\\
\varphi &:: = t \approx	t\mid P\vec{t} \mid  (\varphi \land \varphi) \mid \lnot \varphi \mid \K_i\varphi \mid [x:= t] \varphi 
\end{align*}
where $x\in\Var$, $a\in \Nm$, $P\in \Ps$, and $i\in \Ag$. 
\end{definition}

We call $t\approx t'$ and $P\vec{t}$ \textit{atomic} formulas.  We use the usual abbreviations $\lor, \to, \hK_i$, $\lr{x:=t}$, and write $\Kv_ia$ for $[x:=a]\K_i(x\approx a)$. As we discussed in the introduction, $\Kv_ia$ says the agent $i$ knows the value of $a$. Based on the semantics to be given later, $\Kv_i$ is indeed the same know-value modality discussed in \cite{Plaza89:lopc,WF13}. 

Following \cite{Wangnames}, we define the free and bound occurrences of a variable in a $\ELAS$-formula by viewing $[x:=t]$ in $[x:=t]\phi$ as a quantifier binding $x$ in $\phi$. We call $x$ a \textit{free variable} in $\phi$ if there is a \textit{free occurrence} of $x$ in $\phi$. Formally the set of free variables $\Fv(\phi)$ in $\phi$ is defined as follows: 
\begin{center}
	\begin{tabular}{ll}
		$\Fv(P\vec{t})=\Va(\vec{t})$& $~~\Fv(t\approx t')=\Va(t)\cup \Va(t')$\\
		$\Fv(\neg \phi)=\Fv(\phi)$ & $~~\Fv(\phi\land \psi)=\Fv(\phi)\cup \Fv(\psi)$ \\
		$\Fv(\K_i \phi)=\Fv(\phi)$ & $~~\Fv([x:=t]\phi)=(\Fv(\phi)\setminus \{x\})\cup \Va(t) $ \\
	\end{tabular}
\end{center}
where $\Va(\vec{t})$ is the set of variables in $\vec{t}$. We use $\phi[y\slash x]$ to denote the result of substituting $y$ for all the free occurrences of $x$ in $\phi$, and say $\phi[y\slash x]$ is \textit{admissible} if all the occurrences of $y$ by replacing free occurrences of $x$ in $\phi$ are also free in $\phi[y\slash x]$. It is showed in \cite{Wangnames} that $[x:=t]$ indeed behaves like a quantifier via a translation to a 2-sorted first-order logic. 

The models are simply first-order Kripke models. 
\begin{definition}[Models]
	A (constant domain) Kripke model $\M$ is a tuple  $\lr{W, D, R, \rho, \eta}$ where: 
	\begin{itemize}
		\item $W$ is a non-empty set of possible worlds.
		\item $D$ is a non-empty set of objects, called the \emph{domain} of $\M$. 
		\item $R: \Ag \to 2^{W\times W}$ assign a binary relation $R(i)$ (also written $R_i$) between worlds, to each agent $i$.
		\item $\rho:\Ps\times W\to \bigcup_{n\in \omega}2^{D^n}$ assigns an $n$-ary relation over $D$ each $n$-ary predicate $P$ at each world.
		\item $\eta:\Nm\times W\to D$ assigns an object to each name $a\in\Nm$ at each world $w$.
		\end{itemize}
		Given $\M$, we refer to its components as $W_\M, D_\M, R_\M, \rho_\M, \eta_\M$. A \emph{pointed Kripke model} is a triple $\M,w,\sigma$, where $w\in W_\M$ and  $\sigma:\Var\to D_\M$ assigns an object to every variable. Given $\M$, and a world $w$, $\sigma$ can be lifted to $\sigma_w$ over all the terms $t$ such that $\sigma_w(a)=\eta(a,w)$ for names. An \emph{epistemic} model is a model where the relations are equivalence relations.  
\end{definition}
Note that the names in $\Nm$ and predicates are non-rigid designators. 

\begin{definition}[Semantics]\label{2.6}
The truth conditions are given with respect to $\M,w,\sigma$: 
	$$\begin{array}{rcl}
	\hline
	\M, w, \sigma\vDash t\approx t' &\Leftrightarrow & \sigma_w(t)=\sigma_w(t') \\ 
	\M, w, \sigma\vDash P(t_1\cdots t_n) &\Leftrightarrow & (\sigma_w(t_1), \cdots, \sigma_w(t_n))\in \rho(P,w)  \\ 
	\M, w, \sigma\vDash \neg\phi &\Leftrightarrow&   \M, w, \sigma\nvDash \phi \\ 
	\M, w, \sigma\vDash (\phi\land \psi) &\Leftrightarrow&  \M, w, \sigma\vDash \phi \text{ and } \M, w, \sigma\vDash \psi \\ 
	\M, w, \sigma\vDash \K_i \phi &\Leftrightarrow& \M, v, \sigma \vDash\phi \text{ for all $v$ s.t.\ $wR_iv$}\\
	\M, w, \sigma\vDash [x:=t]\phi &\Leftrightarrow& \M, w, \sigma[x\mapsto \sigma_w(t)]\vDash \phi\\
	\hline 
	\end{array}$$
		where $\sigma[x\mapsto \sigma_w(t)]$ denotes an assignment that is the same as $\sigma$ except for mapping $x$ to $\sigma_w(t)$.
\end{definition}

Now we can check the derived semantics for $\Kv_i$:
\begin{align*}
    &\M, w, \sigma\vDash \Kv_i a\\
    &\Leftrightarrow \M, w, \sigma \vDash [x:=a]\K_i (x\approx a)\\
&\Leftrightarrow \M, v, \sigma[x\mapsto \sigma_w(a)] \vDash x\approx a \text{ for all $v$ s.t.\ $wR_iv$}\\
&\Leftrightarrow \sigma_w(a) = \sigma_v(a) \text{ for all $v$ s.t.\ $wR_iv$}
\end{align*}
Over reflexive models we have the semantics in \cite{WF13}: $$\M, w, \sigma\vDash \Kv_i a \Leftrightarrow \sigma_v(a) = \sigma_{v'}(a) \text{ for all $v, v'$ s.t.\ $wR_iv$ and $ wR_iv'$}.$$ 
\begin{example}\label{ex.nodicto}
Consider the following model $\M$ as a simple example with two worlds $s,t$, a signature that contains only the unary predicate $P$, one agent $i$, a domain with two objects $o_1$, $o_2$, a $\rho$ such that $\rho(P, s) = \{o_1\}, \rho(P, t) = \{o_2\}$, and an $\eta$ depicted below by abusing the symbol $\approx$:  
$$\xymatrix{
\underline{s:a\approx o_1, b\approx o_2  }\ar@(ur,ul)|{i}\ar@{<->}[rr]|{i}&{ }&{t: a\approx o_2, b\approx o_1} \ar@(ur,ul)|{i}
}
$$
\end{example}
In the above example, agent $i$ has the \textit{de dicto} knowledge that $P(a)$: $\M,s,\sigma \models \K_i P(a)$ for any $\sigma$, since the formula does not contain any free variable. However, note that $\M,s, \sigma \models \lnot [x:= a] \K_i P(x)$, i.e., in the actual world $s$ (underlined), agent $i$ does not have the \textit{de re} knowledge that the object that $a$ denotes (object $o_1$) has property $P$. Although the agent knows all the propositional facts regarding the property $P$, it still has the \textit{de re} ignorance. Further note that no closed formula involving just $P$ can distinguish states $s$ and $t$. A \textit{de re} update is needed for the agent to learn that state $t$ is not the actual world. 

Adapting the proofs in \cite{Wangnames,WangWeiSel}, it is not hard to show the following, which we leave for the full version of the paper: 
\begin{itemize}
    \item $[x:=t]$ cannot be eliminated  in $\ELAS$ qua expressivity. 
    \item $\ELAS$ is decidable over arbitrary and reflexive models;
    \item $\ELAS$ is undecidable over S5 models. 
\end{itemize}
The undecidability can be shown by coding Fitting's undecidable logic \textbf{S5}$\lambda_=$ where instead of the assignment operator, the $\lambda$-abstraction $\lr{\lambda x.\phi}$ is used to handle the distinction between \textit{de dicto} and \textit{de re}, e.g., $\lr{\lambda x.\Box \lr{\lambda y. y\approx x} (c)}(c)$ says $c$ is rigid, which is  equivalent to our $[x:=c]\K[y:=c]x\approx y.$ Note that our assignment operator is much easier to read compared to the $\lambda$-abstraction.

\subsection{Axiomatization}

Based on the axioms in \cite{Wangnames}, we proposed the following proof system $\SELAS$. 
\begin{center}
 \begin{tabular}{ll}
\hline
Axiom Schemas &\\
\hline
\TAUT & \text{ all the instances of tautologies}\\
\DISTK &$ \K_i(\phi\to\psi)\to (\K_i\phi\to \K_i\psi$)\\
	\AxId &$t\approx t$\\
\AxSym	& $t\approx t' \lra t'\approx t$\\
\AxTranseq	& $t\approx t'\land t'\approx t''\to t\approx t''$\\
	\AxSUBAS &$t\approx t' \to ([x:=t] \phi \lra [x:=t']\phi)$\\
	\AxSUBP &$\vec{t}\approx \vec{t'} \to (P\vec{t}\lra P\vec{t'})$ \\
			\AxRGDP   &$x\approx y \to \K_i x\approx y$\\
			\AxRGDN        &$x\not\approx y \to \K_i  x\not\approx y$\\
			\AxKAS& $[x:=t](\phi\to \psi)\to ( [x:=t]\phi\to[x:=t]\psi)$ \\
						\AxDETAS &$\lr{x:=t}\phi\to [x:=t]\phi$\\
			\AxEXEAS &$\lr{x:=t}\top$\\
			\AxEFAS  &$[x:=t]x\approx t$ \\
			\AxSUBtoAS &$\phi[y\slash x]\to [x:=y]\phi$ \quad ($\phi[y\slash x]$ is admissible)   \\

\hline
Rules&\\
\hline
\GENK &$\dfrac{\phi}{\K_i\phi}$\\
\texttt{MP}& $\dfrac{\phi,\phi\to\psi}{\psi}$\\
\NECAS&$\dfrac{\vdash\varphi\to \psi}{\vdash \varphi \to [x:=t] \psi} \quad \text{($x$ is not free in $\phi$)}$\\
\hline
\end{tabular}
\end{center}
\noindent where in $\AxSUBP$, $\vec{t}\approx \vec{t'}$ is the abbreviation of the conjunction of point-wise equivalences for sequences of terms $\vec{t}$ and $\vec{t'}$ such that $|\vec{t}|=|\vec{t'}|$.  The system $\SELASf$ is defined as $\SELAS$ together with the usual S5 schemata for $\K_i$: $\AxTrK: \K_i\phi\to \phi$, $\AxTrans: \K_i\phi\to\K_i\K_i\phi$, and $\AxEuc: \neg \K_i\phi\to \K_i\neg\K_i\phi$.

Note that $\AxId, \AxSym, \AxTranseq$ captures the nature of equality; $\AxSUBP$ and $\AxSUBAS$ regulate the substitution that we can safely do; $\AxRGDP$ and  $\AxRGDN$ describe that the variables are rigid;  $\AxKAS$, $\AxDETAS$, $\AxEXEAS$ capture that the assignment operator is a self-dual normal modality (the necessitation rule for $[x:=t]$ is a special case of $\NECAS$); $\AxEFAS$ characterizes the effect of the assignment operator; $\AxSUBtoAS$ and $\NECAS$ are the  counterparts of the usual axiom and rules of the first-order quantifier.
\begin{remark}
Comparing to \cite{Wangnames}, we do not have the special axioms to handle the term-modalities, and the S5 axioms for the epistemic operators are standard. 
\end{remark}
Given that $\K_i$ and $[x:=t]$ are normal modalities, we can show that the rule of replacement of equals is an admissible rule in the systems $\SELAS$ and $\SELASf$.
\begin{thm}[Soundness]\label{thm.sound}
System $\SELAS$ is sound over Kripke models and $\SELASf$ is sound over epistemic models. 
\end{thm}
The following are handy for the completeness proof. 
\begin{prop}
	The following are derivable from the system:
	$$
	\begin{array}{ll}
		\mathtt { DBASEQ } & \langle x:=t\rangle \varphi \leftrightarrow[x:=t] \varphi \\
		\mathtt{ CNECAS } & \displaystyle\frac{\vdash \varphi \rightarrow \psi}{\vdash[x:=t] \varphi \rightarrow \psi} \quad(x \notin F v(\psi)) \\ 
		\mathtt { EAS } & {[x:=t] \varphi \leftrightarrow \varphi(x \notin \Fv(\varphi))} \\
		\mathtt { SUBASEQ } & \varphi[y / x] \leftrightarrow[x:=y]\phi \quad \text{ given } \varphi[y / x] \text { is admissible } \\
		\mathtt { NECAS }^{\prime} & \displaystyle\frac{\vdash \varphi}{[x:=t] \varphi}
	\end{array}
	$$
\end{prop}
\begin{proof} $\mathtt{ DBASEQ}$ is based on $\mathtt{DETAS}$ and  $\mathtt{DAS}$. $\mathtt{CNECAS}$ is due to $\mathtt{NECAS}$ and $\mathtt{DBASEQ}$ for contrapositive. $\mathtt{EAS} $ is based on $\mathtt{NECAS}$ and $\mathtt{CNECAS}$ (taking $\psi=\varphi$). $\mathtt{SUBASEQ}$ is due to the contrapositive of $\mathtt{SUB2AS}$ and $\mathtt{DBASEQ}$. $\mathtt{NECAS}'$ is special case for $\mathtt{NECAS}$.
\end{proof}
We can also rename the bound variables as shown in \cite{Wangnames}. 
\begin{prop}[Relettering]
Let $z$ be fresh in $\varphi$ and $t$, then 
	$\vDash[x:=t] \varphi \leftrightarrow[z:=t] \varphi[z / x].$
\end{prop}

The completeness of $\SELAS$ and $\SELASf$ can be proved by an adaptation of the corresponding (highly non-trivial) completeness proof in \cite{Wangnames} without the treatment for the term-modalities in \cite{Wangnames}. We omit the proof and leave it to the full version due to page limitation.  
\begin{theorem}\label{thm.compelas}
$\SELAS$ is strongly complete over arbitrary models and $\SELASf$ is strongly complete over epistemic models. 
\end{theorem}

\section{Adding public announcement}\label{sec.pa}
In this section, we develop a public announcement logic based on the language $\ELAS$. We will show that as in the case of standard $\PAL$, the announcement operator can be eliminated. 
\subsection{Language and semantics}
\begin{definition}[Language of $\PALAS$]The language of Public Announcement Logic with Assignments ($\PALAS$) is defined by adding the announcement operator to $\ELAS$: 
\begin{align*}
t&::= x\mid a\\
\varphi &:: = t \approx	t\mid P\vec{t} \mid  (\varphi \land \varphi) \mid \lnot \varphi \mid \K_i\varphi \mid [x:= t] \varphi  \mid [!\phi]\phi
\end{align*}
where $x\in\Var$, $a\in \Nm$, $P\in \Ps$ and $i\in \Ag$. 
\end{definition}
As in the standard $\PAL$, $[\psi]\phi$ intuitively says that if $\psi$ can be truthfully announced then afterwards $\phi$ holds. Besides the usual abbreviations, we also write $\lr{!\phi}$ for $\neg[!\phi]\neg$. 

With this simple addition, we can capture the \textit{de re} updates of publicly announcing the actual value of $a$ by $[x:=a][!x\approx a]$. In the following, we will also write $[!a]\phi$ for $[x:=a][!x\approx a]\phi$ when $x$ is not free in $\phi$. This is essentially the $[a]$ operator introduced in \cite{GvEW16}.

We can actually do \textit{much more} beyond announcing the value of $a$. For example, the \textit{de re} announcement that the reference of $a$ does have the property $P$ is expressed by $[x:=a][!Px]$. This will give the \textit{de re} knowledge that object $o_1$ has property $P$ to the agent in Example \ref{ex.nodicto}.


The semantics of the announcement operator is essentially the same as in standard $\PAL$. 
\begin{definition}[Semantics]
	$$\begin{array}{rcl}
	\hline
	\M, w, \sigma\vDash [!\psi]\phi &\Leftrightarrow& \M, w, \sigma\vDash \psi \text{ implies }\M|^\sigma_{\psi}, w, \sigma \vDash \phi\\
	\hline 
	\end{array}$$
 where $\M|^\sigma_{\psi}$ is the submodel of $\M$ restricted to the $\psi$ worlds in $\M$, i.e., 
	$\M|_{\psi}^{\sigma}=\{W', D_\M, R', \eta'\}$ such that $W'=\lr{v\mid \M,v,\sigma\vDash\psi}$, $R'_i=R_i|_{W'\times W'}$ and $\eta'=\eta|_{W'}$.
\end{definition}
Now we can check the induced semantics of $[!a]$:
$$\M, w, \sigma\vDash [!a]\phi \Leftrightarrow \M, w, \sigma \vDash [x:=a][!x\approx a]\phi \Leftrightarrow \M, w, \sigma[x\mapsto \sigma_w(a)] \vDash [!x\approx a]\phi \Leftrightarrow \M|^{\sigma[x\mapsto \sigma_w(a)]}_{x\approx a}\vDash \phi$$
where $\M|^{\sigma[x\mapsto \sigma_w(a)]}_{x\approx a}$ is the submodel of $\M$ with all the worlds that share the same value of $a$ as the actual world. This is indeed the semantics given to the \textit{public inspection operator} in \cite{GvEW16}. 
	
With the announcement operator, we can also define the conditional operators  introduced in \cite{WF14,Baltag16} over epistemic models. For example:
\begin{itemize}
 \item $\Kv_i(\phi,c):= \K_i[!\phi]\Kv_ic:$ Agent $i$ would know the value of $c$ given $\phi$.
    \item $\Kv_i(c, d):= \K_i[!c]\Kv_id$: Agent $i$ would know the value of $d$ given the value of $c$, namely, agent $i$ knows how the value of $d$ functionally depends on the value of $c$;
    \item $\Kv_i(c, \phi):= \K_i[!c](\K_i\phi\lor \K_i\neg\phi)$: Agent $i$ would know the truth value of $\phi$ given the value of $c$, i.e., agent $i$ knows how the truth value of $\phi$ depends on the value of $c$;
    \item  $\Kv_i(\psi, \phi):= \K_i([!\psi](\K_i\phi\lor \K_i\neg\phi)\land [!\neg \psi](\K_i\phi\lor \K_i\neg\phi))$: Agent $i$ knows how the truth value of $\phi$ depends on the truth value of $\psi$. 
\end{itemize}

Based on the semantics, it is not hard to show the axioms of perfect recall and no miracles are still valid, which form the foundation for the reduction of the announcement operator to be introduced (cf. \cite{WC12}).
\begin{proposition}
The following are valid: 
\begin{center}
 \begin{tabular}{ll}
\hline
\PR & $\K_i[!\psi]\phi\to [!\psi]\K_i\phi$\\
\NM& $\lr{!\psi}\K_i\phi\to\K_i[!\psi]\phi$\\
\hline
\end{tabular}
\end{center}
\end{proposition}
	
\subsection{Axiomatization}
We define the proof system $\SPALAS$ (\SPALASf) as the proof system obtained by extending $\SELAS$ ($\SELASf$) with the following reduction axioms, which help us to eliminate the announcement operator in $\PALAS$. $\AK$ is essentially the combination of the axioms $\PR$ and $\NM$ (cf. \cite{WC12}). Besides the usual reduction axioms for \PAL, we have a new axiom $\ASS$.  
\begin{center}
 \begin{tabular}{ll}
\hline
Axiom Schemas &\\
\hline

\ATOM & $[!\psi]p\lra(\psi\to p)$  \text{ (if $p$ is atomic)}\\
\NEG & $[!\psi]\neg \phi\lra(\psi\to \neg [!\psi]\phi)$\\
\CON & $[!\psi](\phi\land\chi)\lra([!\psi]\phi\land [!\psi]\chi)$\\
\AK & $[!\psi]\K_i\phi\lra (\psi\to \K_i [!\psi]\phi)$\\
\CCOM & $[!\psi][!\chi]\phi\lra [!(\psi\land [!\psi]\chi)]\phi$\\
\ASS&$[!\psi][x:=t] \phi \lra [z:=x][x:=t] [!\psi[z/x]] \phi$ \\&($z$ does not occur in $[!\psi][x:=t] \phi$)\\
\hline
\end{tabular}
\end{center}

\begin{theorem}
$\SPALAS$ is sound over arbitrary models. 
\end{theorem}
\begin{proof}
The validity of the first five reduction axioms is as in the standard $\PAL$. We only focus on the last one, $\ASS$, which is about switching the assignment operator and the announcement operator. Note that in $\ASS$, $z$ is fresh in $\varphi,\psi$ and $z\not= t$, therefore $\psi[z/x]$ is always admissible. We first prove the following claim: 
\begin{claim} For any $v$ in $\M$:
$$\M,v,\sigma\vDash\psi \iff \M,v,\sigma \vDash [z:=x][x:=t]\psi[z/x]$$
\end{claim}
\begin{claimproof}
Since $z$ is fresh, and there is no free $x$ in $\psi[z/x]$, we have for any $v,u$ in $\M$:
\begin{align*}
    &\M,v,\sigma\vDash\psi\\
\iff&\M,v,\sigma[z\mapsto\sigma(x)]\vDash\psi[z/x]\\
\iff& \M,v,\sigma[z\mapsto\sigma(x)][x\mapsto\sigma_{u}(t)]\vDash \psi[z/x] \qquad (\star).
\end{align*}
Let $\sigma^*=\sigma[z\mapsto\sigma(x)]$. Since $t\not=z$, and since changing $\sigma$ does not affect $\eta$ on $u$, we have for any $u$ in $\M$ $$\sigma_u(t)= \sigma^*_u(t) \quad (\dagger)$$ no matter whether $t$ is a variable or a name.  Therefore we have for any $u$ in $\M$:
$$\sigma[z\mapsto\sigma(x)][x\mapsto\sigma_u(t)]=\sigma[z\mapsto\sigma(x)][x\mapsto\sigma^*_u(t)]$$

Now from ($\star$) we have : 
$$\M,v,\sigma\vDash\psi\iff  \M,v,\sigma[z\mapsto\sigma(x)][x\mapsto\sigma^*_u(t)]\vDash \psi[z/x] \quad (\ddagger) $$

In particular, taking $u=v$ in ($\ddagger$) gives us the proof for the claim according to the semantics.
\end{claimproof}

Now consider the following two cases:\\

\indent \textbf{(Case I)} If $\M,w,\sigma\not\vDash\psi$, then 
$\M,w,\sigma\vDash[!\psi][x:=t] \phi$ is trivially true. By the above claim, $\M,w,\sigma \nvDash [z:=x][x:=t]\psi[z/x]$. Thus $\M,w,\sigma\vDash [z:=x][x:=t] [!\psi[z/x]] \phi.$\\

\textbf{(Case II)} If $\M,w,\sigma\vDash\psi$, by the above claim, $\M,w,\sigma \vDash [z:=x][x:=t]\psi[z/x]$. According to the semantics we need to show (1) iff (2) below : 
$$(1)\ \M|^\sigma_\psi,w,\sigma[x\mapsto \sigma_w(t)]\vDash\varphi$$ 
$$(2)\ \M|^{\sigma[z\mapsto \sigma_{w}(x)][x\mapsto\sigma^*_w(t)]}_{\psi[z/x]},w,\sigma[z\mapsto \sigma_{w}(x)][x\mapsto\sigma^*_w(t)]\vDash\varphi$$


Note that $\sigma(x)=\sigma_w(x)$ by definition. Now taking $u=w$ in ($\ddagger$) immediately shows that $\M|^{\sigma[z\mapsto\sigma_w(x)][x\mapsto\sigma^*_w(t)]}_{\psi[z/x]}$ is exactly the same model as  $\M|^\sigma_\psi$. Now we only need to consider whether the difference between $\sigma[x\mapsto \sigma_w(t)]$ and $\sigma[z\mapsto \sigma_{w}(x)][x\mapsto\sigma^*_w(t)]$ matters for the truth value of $\phi$. Note that $z$ does not occur in $\phi$ and by ($\dagger$) $\sigma_w(t)= \sigma^*_w(t)$, therefore the above difference in $\sigma$ does not affect the truth value of $\phi$. It follows that (1) iff (2), and this completes the proof.

\end{proof}
With the formulas above, we can translate $\PALAS$-formulas to $\ELAS$-formulas and eliminate the public announcement operators.

Based on the above theorem, we can define a translation $\trs$ to eliminate the announcement operators as in the standard \PAL\ using the left-to-right direction of the reduction axioms  \cite{DELbook,WC12}, and the following extra clause (where $z$ is fresh): 
$$\trs([!\psi][x:=t] \phi)= [z:=x][x:=t] \trs([!\psi[z/x]] \phi)$$
It is not hard to show that the translation preserves the equivalence of formulas.
\begin{proposition} \label{prop.exp}
For all $\phi\in\PALAS$: 
$\vDash \phi\lra \trs(\phi)$
\end{proposition}
\begin{theorem}
$\SPALAS$ is complete over arbitrary models.
\end{theorem}
\begin{proof}The proof is done by the following reduction.
 $$\vDash\phi\implies\vDash \trs(\phi)\stackrel{}{\implies}\vdash_{\SELAS} \trs(\phi)\implies\vdash_{\SPALAS} \trs(\phi)\stackrel{}{\implies}\vdash_{\SPALAS}\phi$$
 The first step is due to Proposition \ref{prop.exp}. The second step is due to Theorem \ref{thm.compelas}. The third step is due to the fact that $\SPALAS$ is an extension of $\SELAS$, and the last step is due to the reduction axioms in the system that you can show $\vdash_\SPALAS \trs(\phi)\lra \phi.$ 
\end{proof}
Similarly we can show that:
\begin{theorem}
$\SPALASf$ is complete over epistemic models. 
\end{theorem}

\section{Adding Event Models}\label{sec.em}
In this section, we generalize the public announcements to event models proposed in \cite{BMS}. We first consider the event models without factual changes. 
\subsection{Language and Semantics}
\begin{definition}[Event model]
An event model $\E$ with respect to a given language $\mathbf{L}$ is a triple: $\lr{E, \tot, Pre}$ where:  
\begin{itemize}
\item $E$ is a finite non-empty set of events;
\item $\tot:\Ag\to 2^{E\times E}$ assigns a relation to each agent;
\item $Pre: E\to \mathbf{L}$ assigns each event a precondition formula.
\end{itemize}
A pointed event model $\E, e$ is an event model with a designated event. An epistemic event model is an event model where the accessibility relations are equivalence relations. 
\end{definition} 

We often write $\tot_i$ for $\tot(i)$ to denote the relation for $i$. 

\begin{definition}[Update product $\otimes$]Given $\Nm$, $\Ag$, a model $\M=\lr{W, D, R, \eta}$, an assignment $\sigma$,  and an event model $\E=\lr{E, \tot, Pre}$ with respect to a given language $\mathbf{L} $, the updated model $(\M\otimes \E)^\sigma$ is a tuple $\lr{W', D, R', \rho', \eta'}$ where:
\begin{itemize}
\item $W'=\{(w, e)\mid \M,w,\sigma\vDash Pre(e) \}$; 
\item $(s,e)R'_i(s',e')$ iff $sR_i s'$ and $e\tot_ie'$;  
\item $\eta'(a,(w,e))=\eta (a, w)$;
\item $\rho'(P,(w,e))=\rho(P,w)$.
\end{itemize} 
\end{definition}
Note that $\sigma$ is necessary in defining the updated model. 

\begin{definition}[Language of \DELAS]
The language of Dynamic Epistemic Logic with Assignments (\DELAS) is defined below:  
\begin{align*}
t&::= x\mid a\\
\varphi& :: = t \approx	t \mid P\vec{t}\mid (\varphi \land \varphi) \mid \lnot \varphi \mid \K_i\varphi \mid [x:= t] \varphi \mid  [\E,e]\phi
\end{align*}
where $x\in\Var$, $a\in \Nm$, $P\in\Ps$, $i\in \Ag$, and $\E,e$ is a pointed event model w.r.t. \DELAS.\footnote{The event model and formulas are defined by a mutual induction, cf. \cite{BMS}. } 
\end{definition}
As in \cite{BMS}, we can compose two event models into one. 
\begin{definition}[Composition of event models]
Let $\E=\lr{E,\tot,Pre}$ and $\E'=\lr{E',\tot',Pre'}$ be two event models. Then the composition of $\E$ and $\E'$ is $\E\circ\E'=\lr{E'',\tot'', Pre''}$ where 
\begin{itemize}
    \item $E''= E\times E'$
    \item $(e,e')\tot_i''(f,f')\iff e\tot_i f$ and $e'\tot_i' f'$
    \item $Pre''(e,e')=Pre(e)\wedge[\E,e] Pre'(e')$
\end{itemize}
The composition of two pointed model $\E,e$ and $\E',e'$(denoted as $(\E,e)\circ(\E',e')$) is defined as the pointed model $\E\circ\E',(e,e')$.
\end{definition}
The semantics is as in the standard event-model $\DEL$.
\begin{definition}[Semantics]
We give the truth condition for the event updates. 
$$
\begin{array}{|rcl|}
\hline
\M,w,\sigma\vDash [\E, e]\phi &\Leftrightarrow& \M,w, \sigma\vDash Pre(e) \Rightarrow
\M\otimes\E,(s,e),\sigma\vDash \phi\\
\hline
\end{array}
$$
\end{definition}

With event models, we can capture  non-trivial \textit{de re} dynamics. For example, agent $1$ is told a password with agent $2$ around, but agent 2 is not sure whose password it is: it could be Cindy's (c) or Dave's (d). The following event model captures such an event $\E$ (with precondition specified):
$$\xymatrix{
\underline{e: x\approx c }\ar@(ur,ul)|{1,2}\ar@{<->}[rr]|2&{ }&{f: x\approx d} \ar@(ur,ul)|{1,2}\\ 
}
$$
The underlining event is the real event. Suppose initially agent 1 and agent 2 have no idea about the possible passwords of $c$ and $d$ (thus think all the natural numbers are possible), the (infinite) initial model $\M,s$ may look like below: 

$$\xymatrix{
\underline{s: c\approx 12, d\approx 34}\ar@(ur,ul)|{1,2}\ar@{<->}[r]|{1,2}& {t: c\approx 4, d\approx 12}\ar@(ur,ul)|{1,2}\ar@{<.>}[r] & \cdots
}
$$
According to the semantics, we can verify 
$$\M,s\vDash [x:=c] [\E,e](\Kv_1 c\land \neg \Kv_1 d\land  \neg \Kv_2 c\land \neg\Kv_2 d\land  \K_2(\Kv_1c \lor \Kv_1 d)).$$

As a variant mentioned in the introduction, we can capture the event where agent 1 and agent 2 are told two numbers (the passwords of $c$ and $d$) such that agent 1 knows which is which but agent 2 does not know it.   
$$\xymatrix{
\underline{e: x\approx c, y\approx d }\ar@(ur,ul)|{1,2}\ar@{<->}[rr]|2&{ }&{f: x\approx d, y\approx c} \ar@(ur,ul)|{1,2}\\ 
}
$$
We can verify:
$$\M,s \vDash [x:=c][y:=d] [\E,e](\Kv_1 c\land \Kv_1 d\land  \neg \Kv_2 c\land \neg\Kv_2 d \land  \K_2(\Kv_1c \land \Kv_1 d)).$$

\subsection{Axiomatization}
We define the proof system $\SDELAS$ ($\SDELASf$) as the proof system obtained by extending $\SELAS$ ($\SELASf$) with the following reduction axioms:
\begin{center}
 \begin{tabular}{ll}
\hline
Axiom Schemas &\\
\hline
\UATOM & $[\E,e]p\leftrightarrow(Pre(e)\rightarrow p)$ \text{ ($p$ is atomic)}\\
\UNEG & $[\E,e]\neg\varphi\leftrightarrow (Pre(e)\rightarrow\neg[\E,e]\varphi)$\\
\UCON & $[\E,e](\varphi\land\psi)\leftrightarrow([\E,e]\varphi\land[\E,e]\psi)$\\
\UAK & $[\E,e]\K_i\varphi\leftrightarrow(Pre(e)\rightarrow\bigwedge_{e\tot_i f}[\E,f]\varphi)$\\
\UCCOM & $[\E,e][\E',e']\varphi\leftrightarrow[\E\circ \E',(e,e')]\varphi$\\
\UASS& $[\mathcal{E},e][x:=t]\varphi\leftrightarrow[z:=x][x:=t][\mathcal{E}',e]\varphi$\\
\hline
\end{tabular}
\end{center}
where in \UASS, $z$ does not occur in $\varphi$, $t$ or in $Pre_\E(e)$ for all $e\in\mathcal{E}$. $\mathcal{E}'$ is an event model with the same domain and relations as $\E$ and for all $f$ in $\E$, $Pre_{\E'}(f)=Pre_\E(f)[z/x]$. 

Note that when $\E$ is a singleton model with precondition $\psi$ then $\UASS$ coincides with $\ASS$. 

\begin{theorem}
$\SPALAS$ is sound over arbitrary models. 
\end{theorem}
\begin{proof}
We only need to check $\UASS$. The proof is similar to the proof of the validity of $\ASS.$ The idea is to use a fresh variable to store the initial value of $x$, and replace the free occurrences of $x$ by $z$ in the preconditions in $\E$ such that all these preconditions will be evaluated according to the initial value of $x$. Note that the reduction also depends on the fact that the update itself does not change the value of $t$ or $x$.
\end{proof}

Similarly, as in the previous section, we have the completeness:
\begin{theorem}
$\SDELAS$ is complete over arbitrary models and $\SDELASf$ is complete over epistemic models. 
\end{theorem}

\subsection{Adding factual changes}
Finally, let us consider event models with factual changes inspired by  \cite{LCC}.

\begin{definition}[Event model with factual changes]
An event model with factual changes $\E$ w.r.t.\ language $\mathbf{L}$ is a tuple: $\lr{E, \tot, Pre, Pos}$ where $E, \tot, Pre$ are defined as before, and
\begin{itemize}
\item $Pos: \Nm\times E\to \Nm$ is a function that maps all but finite names to themselves.
\end{itemize}
\end{definition} 
Intuitively, a post condition changes the value of one name to the value of another one, e.g., an event may switch the value of $c$  and $d$ by setting $Pos(c,e)=d$ and $Pos(d, e)=c$.\footnote{It is more interesting to have function symbols in the language to change the value of $a$ to the value of a term, which we leave for future work.}

Accordingly, we also incorporate the factual change in the definition of the update: 
\begin{definition}[Update product with factual change]Given $\Nm$, $\Ag$, a model $\M=\lr{W, D, R, \rho, \eta}$, an assignment $\sigma$,  and an event model $\E=\lr{E, \tot, Pre, Pos}$ w.r.t. $\mathbf{L} $, the updated Kripke model $(\M\otimes \E)^\sigma$ is a tuple $\lr{W', D, R', \rho', \eta'}$ where: $W', D, R', \rho'$ are as before, and 
\begin{itemize}
\item $\eta'(a,(w,e))=\eta (Pos(a,e), w)$
\end{itemize} 
\end{definition}
We will show that with factual changes, $[\E,e]$ can still be eliminated. 

Given an event model with postconditions $\E$, we first lift the $Pos_\E$ to the function $Pos^+_\E: (\Nm\cup\Var)\times E\rightarrow \Nm\cup\Var$ where for all $x\in\Var$ and $e\in E$, $Pos_\E^+(x,e)=x$ and for all $a\in\Nm$, $Pos_\E^+(a,e)=Pos_\E(a,e)$. 

Now we can state the new reduction axiom.

\begin{center}
 \begin{tabular}{ll}
\hline
\UASS'& $[\E,e][x:=t]\varphi\leftrightarrow[z:=x][x:=Pos^+_\E(t, e)][\E',e]\varphi$\\
\hline
 \end{tabular}
 \end{center}
where $z$ is fresh. $\mathcal{E}'$ is defined as before with the new component $Pos_{\E'}(a,e)=Pos_{\E}(a,e)$, i.e., the postcondition is unchanged.  

It is not hard to verify that $\UASS'$ is valid, and we can use it to give a complete axiomatization as before given the event models with factual changes.

\section{Future directions}\label{sec.con}
In this work, we propose a lightweight dynamic epistemic framework to capture \textit{de re} knowledge and updates. In particular, $\ELAS$ can be viewed as a more powerful alternative to the standard epistemic logic, which can also pre-encode \textit{de re} dynamics. There are numerous directions for further work, inspired by the large body of research on the standard (dynamic) epistemic logic. Here we just list a few. 
\begin{itemize}
\item As in \cite{WangWeiSel}, we can try to add function symbols and allow varying domain in the model.
\item Adding the usual common knowledge operator will create complications in axiomatization as in the case of \DEL, but the \textit{de re} common knowledge comes for free.
\item As in \cite{LCC}, we can try to build a more general framework based on \textit{dynamic logic}. It makes particular sense since the assignment operator is actually from dynamic logic.  
\item We can develop the counterparts of the non-reductive axiomatizations in \cite{WC12,WA13}.
\item We can try to find the boundary of the decidability given different frame conditions. We know \texttt{S5} is bad but \texttt{T} is good \cite{WangWeiSel}, what about \texttt{S4}?  
\item It is also interesting to see whether our framework can capture all the intuitive \textit{de re} updates. We think the answer is negative, e.g., it seems hard to capture the private announcement of some value using finite event models. It may lead to further extensions of the framework. 
\end{itemize}
\paragraph{Acknowledgement} Michael Cohen thanks the Pre-Doctoral Fellowship Program of Stanford Center at Peking University, which made this joint work possible. Yanjing Wang acknowledges the support of NSSF grant 19BZX135. 

\bibliographystyle{eptcs}
\bibliography{sgwyj_KH}

\end{document}